\documentclass{ifacconf}

\usepackage{graphicx}
\usepackage{natbib}
\usepackage{amsmath,amssymb,bm}
\usepackage{booktabs}
\usepackage{multirow}
\usepackage{algorithm}
\usepackage{algpseudocode}

\newcounter{proposition}
\renewcommand{\theproposition}{\arabic{proposition}}

\newenvironment{proposition}[1][]
  {\refstepcounter{proposition}%
   \par\medskip
   \noindent\textbf{Proposition~\theproposition.%
   \if\relax\detokenize{#1}\relax\else\ (#1)\fi} \itshape}
  {\par\medskip}
  
\begin{document}
\begin{frontmatter}

\title{Fleet-Level Battery-Health-Aware Scheduling for Autonomous Mobile Robots}

\author[First]{Jiachen Li}
\author[First]{Shihao Li}
\author[First]{Jian Chu}
\author[First]{Wei Li}
\author[First]{Dongmei Chen}

\address[First]{University of Texas at Austin,
   Austin, TX 78712 USA (e-mail: jiachenli@utexas.edu, shihaoli@utexas.edu, jian\_chu@utexas.edu, weiwli@austin.utexas.edu, dmchen@me.utexas.edu).}

\begin{abstract}
Autonomous mobile robot fleets must coordinate task allocation and charging under limited shared resources, yet most battery-aware planning methods address only a single robot. This paper extends degradation-cost-aware task planning to a multi-robot setting by jointly optimizing task assignment, service sequencing, optional charging decisions, charging-mode selection, and charger access while balancing degradation across the fleet. The formulation relies on reduced-form degradation proxies grounded in the empirical battery-aging literature, capturing both charging-mode-dependent wear and idle-state-of-charge-dependent aging; the bilinear idle-aging term is linearized through a disaggregated piecewise McCormick formulation. Tight big-$M$ values derived from instance data strengthen the LP relaxation. To manage scalability, we propose a hierarchical matheuristic in which a fleet-level master problem coordinates assignments, routes, and charger usage, while robot-level subproblems—whose integer part decomposes into trivially small independent partition-selection problems—compute route-conditioned degradation schedules. Systematic experiments compare the proposed method against three baselines: a rule-based nearest-available dispatcher, an energy-aware formulation that enforces battery feasibility without modeling degradation, and a charger-unaware formulation that accounts for degradation but ignores shared-charger capacity limits.
\end{abstract}

\begin{keyword}
autonomous mobile robots, battery health, fleet scheduling, charging coordination, task allocation, mixed-integer linear programming
\end{keyword}

\end{frontmatter}

\section{Introduction}
Autonomous mobile robots (AMRs) are widely deployed in industry for transportation and material-handling tasks. In these environments, battery usage directly governs fleet uptime, maintenance costs, and infrastructure loading \citep{fragapane2021}. Most existing research on fleet-planning treats battery level merely as a minimum-feasibility constraint. Although energy-aware multi-robot assignment has been explored \citep{djenadi2022}, long-horizon battery degradation is seldom explicitly considered. Without such consideration, dispatching policies tend to hold robots at a high state of charge (SOC) and apply aggressive charging when the SOC is low. Both practices can accelerate long-term capacity fade.

Lithium-ion battery aging arises from both cycling-induced and calendar-driven mechanisms. Empirical studies show that degradation is strongly influenced by charging current, depth of discharge, temperature, and storage SOC \citep{doyle1993,fuller1994,choi2002,liaw2003,ramadass2003,rong2006,tredeau2009,lam2013,grolleau2014}. For AMR operations, two factors are especially important: fast charging intensifies cycling-related wear, and prolonged idling at elevated SOC accelerates calendar aging. An effective scheduler must therefore manage not only whether a robot has sufficient energy to complete its next task, but also how it charges and how long it idles at a given SOC.

Recent work has begun to address this battery-health-aware scheduling problem. \citet{chu2024acc} studied energy-efficient AMR trajectory planning under configuration constraints. \citet{li2025robust} formulated a single-AMR task-planning problem to minimize both charging- and idle-SOC-related degradations. Their results show that avoiding unnecessary high-SOC idling can significantly reduce battery degradation without sacrificing scheduling performance. However, at the modeling level, such a health-aware scheduling formulation introduces a bilinear structure through the product of SOC and idle-time. Fortunately, McCormick envelopes \citep{mccormick1976} and their piecewise refinements \citep{castro2015} provide a standard linearization approach to obtaining a tractable MILP. Other linearization techniques that could be employed include reformulation-linearization \citep{sherali1992}, piecewise-linear approximations \citep{misener2010}, and outer approximation \citep{bergamini2008}.

Previous battery-health work has focused on isolated robots. For real-world deployment, fleet performance needs to be optimized. Related to fleet scheduling, the MRTA literature has provided taxonomies for coordination problems \citep{korsah2013}, and electric vehicle routing research has optimized charging-station visits along routes \citep{schneider2014,keskin2016}. However, these studies  typically neglect battery degradation. For a fleet, health-aware scheduling becomes a coupled resource-allocation problem: task assignments shape energy trajectories, charging decisions compete for limited chargers, and the repeated dispatch of the same robots concentrates wear. Although energy-aware fleet allocation exists \citep{djenadi2022}, long-term degradation and charger contention remain weakly integrated.

This paper develops a fleet-level, battery-health-aware scheduling methodology  with shared charging resources, jointly determining task assignment, service order, charging decisions, and degradation distribution across the fleet. The fleet-level formulation introduces new coupling constraints through shared chargers and workload distribution that demand dedicated modeling and algorithmic treatment.

The main contributions are:
\begin{enumerate}
\item \emph{Fleet-level degradation-coupled formulation.}
      We formulate a fleet-level MILP that jointly captures task assignment, sequencing, charging-mode selection, and shared-charger coordination under explicit battery-degradation objectives.
      Cross-robot charger-ordering constraints and a degradation-balancing penalty couple per-robot substructures, while arc-specific big-$M$ values derived from instance data yield a substantially tighter LP relaxation than na\"{i}ve bounds.

\item \emph{Hierarchical matheuristic with decomposable subproblems.}
      We design a two-level solution method in which a fleet-level master problem resolves the combinatorial assignment and charger-ordering decisions, and robot-level subproblems optimize route-conditioned battery-health schedules.
      We prove that the subproblem's integer component decomposes into independent partition-selection problems of trivially small size (Propositions~\ref{prop:subproblem-size}--\ref{prop:subproblem-lp}), enabling efficient exact solution of each subproblem.

\item \emph{Comprehensive experimental evaluation.}
      We conduct systematic experiments across a range of fleet sizes, demonstrating that the proposed method reduces total degradation by up to 54\% over rule-based dispatch and achieves near-monolithic solution quality in over \(20{\times}\) lower computation time on the largest instances.
\end{enumerate}

The remainder of the paper is organized as follows. Section~2 presents the fleet-level formulation. Section~3 introduces the solution method. Section~4 describes the experimental results. Section~5 concludes.

\begin{table}[ht]
\small
\setlength{\tabcolsep}{3pt}
\begin{tabular}{@{}ll@{}}
\toprule
\multicolumn{2}{@{}c}{\textbf{Nomenclature}} \\
\midrule
\multicolumn{2}{@{}l}{\textit{Sets and indices}} \\
\midrule
$\mathcal{R}$              & Set of robots, indexed by $r$ \\
$\mathcal{K}$              & Set of tasks, indexed by $k$ \\
$\mathcal{M}$              & Set of chargers, indexed by $m$ \\
$\mathcal{L}_r$            & Admissible charging modes for robot $r$, indexed by $\ell$ \\
$\mathcal{V}^{-},\mathcal{V}^{+}$ & Source/sink augmented node sets \\
$\mathcal{A}$              & Admissible predecessor--successor arc set \\
$\mathcal{P}_S,\mathcal{P}_W$  & Partition index sets for SOC and waiting time \\
$\Delta$                   & Direct-transition index set \\
$\Gamma$                   & Charging-transition index set \\
$\Omega_m$                 & Potential charging sessions on charger $m$ \\
\midrule
\multicolumn{2}{@{}l}{\textit{Task and travel parameters}} \\
\midrule
$a_k,\; b_k$              & Release time, due time of task $k$ \\
$p_k$                      & Service duration of task $k$ \\
$e_k$                      & Energy consumption of task $k$ \\
$\tau^d_{ij},\; e^d_{ij}$ & Direct travel time and energy from $i$ to $j$ \\
$\tau^{\mathrm{to}}_{im},\; e^{\mathrm{to}}_{im}$ & Travel time/energy from node $i$ to charger $m$ \\
$\tau^{\mathrm{from}}_{mj},\; e^{\mathrm{from}}_{mj}$ & Travel time/energy from charger $m$ to node $j$ \\
\midrule
\multicolumn{2}{@{}l}{\textit{Battery and degradation parameters}} \\
\midrule
$S_r^0,\; S_r^{\min},\; S_r^{\max}$ & Initial, minimum, maximum SOC of robot $r$ \\
$c_{r\ell}$                & Charging rate of robot $r$ in mode $\ell$ \\
$\alpha_{r\ell}$           & Charging-aging coefficient (mode $\ell$) \\
$\beta_r$                  & Idle-aging coefficient \\
\midrule
\multicolumn{2}{@{}l}{\textit{Objective weights and big-$M$ constants}} \\
\midrule
$\lambda,\;\mu,\;\rho$    & Weights for charger waiting, tardiness, imbalance \\
$H$                        & Planning-horizon upper bound \\
$M_t,\; M_s$              & Big-$M$ constants for time and SOC constraints \\
\midrule
\multicolumn{2}{@{}l}{\textit{Decision variables}} \\
\midrule
$x_{rk}$                  & $=1$ if task $k$ is assigned to robot $r$ \\
$y_{rij}$                 & $=1$ if robot $r$ transitions directly from $i$ to $j$ \\
$g_{rijm\ell}$            & $=1$ if robot $r$ charges at $m$ in mode $\ell$ between $i,j$ \\
$T_k$                     & Start time of task $k$ \\
$\mathrm{tard}_k$         & Tardiness of task $k$ \\
$s^{\mathrm{in}}_{rk}$    & Arrival SOC of robot $r$ at task $k$ \\
$B_{rijm\ell}$            & Charging start time \\
$q_{rijm\ell}$            & Charger waiting time \\
$t^c_{rijm\ell}$          & Charging duration \\
$w_{rijm\ell}$            & Post-charge waiting time \\
$\bar{s}_{rijm\ell}$      & Post-charge SOC \\
$\ell_{rijm\ell}$         & Auxiliary for bilinear product $\bar{s}\cdot w$ \\
$z^{pq}_{rijm\ell}$       & $=1$ if partition $(p,q)$ is selected \\
$u_{\sigma,\sigma',m}$    & Charger-ordering binary for sessions $\sigma,\sigma'$ on $m$ \\
$A_r,\; A^{\max}$         & Total degradation of robot $r$; fleet maximum \\
\bottomrule
\end{tabular}
\end{table}

\section{Fleet-Level Problem Formulation}
We consider a fleet of battery-powered AMRs operating in a warehouse with transportation tasks and shared chargers. Each task is executed exactly once by one robot, each robot serves its assigned tasks sequentially, and between consecutive tasks a robot may travel directly or visit a charger. The scheduler jointly determines task assignment, service order, charging decisions, charger access, and battery-health operating points. Charging decisions are represented through a finite set of modes, yielding a MILP that captures both charging-mode-dependent and idle-SOC-dependent degradation. Figure~\ref{fig:structure} illustrates this structure.

\begin{figure}[t]
\centering
\includegraphics[width=\columnwidth]{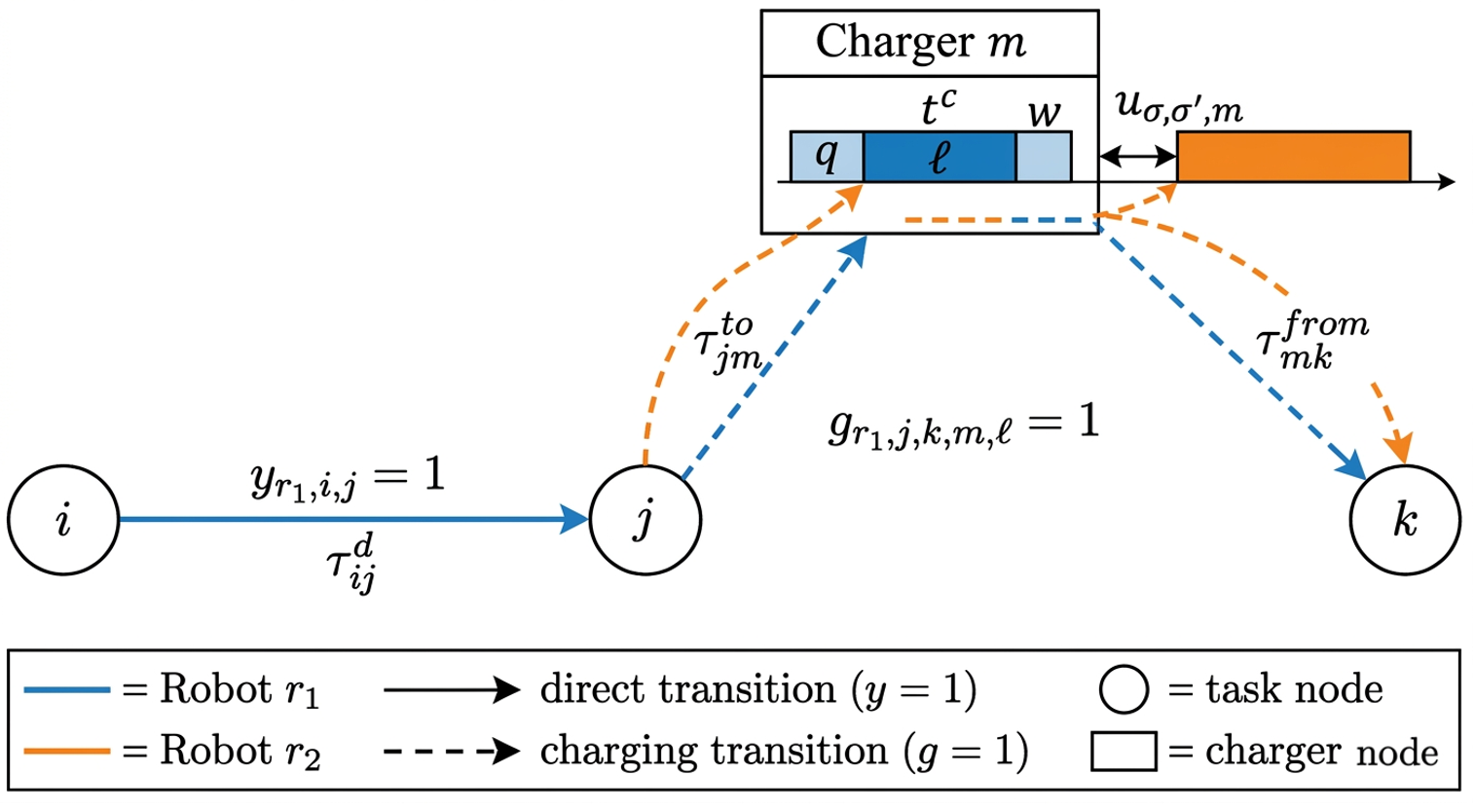}
\caption{Direct and charging transitions with shared-charger non-overlap coupling.}
\label{fig:structure}
\end{figure}

\subsection{Sets, parameters, and decision variables}
Let \(\mathcal{R}\), \(\mathcal{K}\), and \(\mathcal{M}\) denote the sets of robots, tasks, and chargers. Define source and sink dummy nodes \(0\) and \(n+1\), with \(\mathcal{V}^{-}:=\{0\}\cup\mathcal{K}\) and \(\mathcal{V}^{+}:=\mathcal{K}\cup\{n+1\}\). The admissible predecessor--successor pairs are \(\mathcal{A}:=\{(i,j): i\in\mathcal{V}^{-},\, j\in\mathcal{V}^{+},\, i\neq j\}\); charging transitions are defined only for successors \(j\in\mathcal{K}\).

For each robot \(r\in\mathcal{R}\), let \(\mathcal{L}_r\) denote its admissible charging modes, where mode \(\ell\in\mathcal{L}_r\) has charging rate \(c_{r\ell}>0\) and charging-aging coefficient \(\alpha_{r\ell}>0\). Let \(\beta_r>0\) be the idle-aging coefficient. Each task \(k\in\mathcal{K}\) has release time \(a_k\), due time \(b_k\), service duration \(p_k\), and energy consumption \(e_k\). For direct travel from \(i\) to \(j\), let \(\tau^{d}_{ij}\) and \(e^{d}_{ij}\) denote travel time and energy. For travel via charger \(m\), let \((\tau^{\mathrm{to}}_{im}, e^{\mathrm{to}}_{im})\) and \((\tau^{\mathrm{from}}_{mj}, e^{\mathrm{from}}_{mj})\) denote the corresponding times and energies.

Robot \(r\) has initial SOC \(S_r^0\), reserve \(S_r^{\min}\), and maximum \(S_r^{\max}\). Let \(H\) upper-bound all start times and durations, let \(M_t\) and \(M_s\) be big-\(M\) constants for time and SOC constraints, and let \(\lambda,\mu,\rho\ge 0\) weight charger waiting, tardiness, and degradation imbalance. For the piecewise linearization, \(\mathcal{P}_S\) and \(\mathcal{P}_W\) index partitions of the post-charge SOC and waiting-time ranges into intervals \([\underline S_p,\overline S_p]\) and \([\underline W_q,\overline W_q]\).

The direct-transition and charging-transition index sets are
\begin{align}
\Delta := \{(r,i,j):\;& r\in\mathcal{R},\; i\in\mathcal{V}^{-}, \nonumber\\
& j\in\mathcal{K},\; i\neq j\}, 
\label{eq:def-delta}\\
\Gamma := \{(r,i,j,m,\ell):\;& r\in\mathcal{R},\; i\in\mathcal{V}^{-}, \nonumber\\
& j\in\mathcal{K},\; i\neq j, \nonumber\\
& m\in\mathcal{M},\; \ell\in\mathcal{L}_r\}.
\label{eq:def-gamma}
\end{align}

The decision variables are as follows. Binary variables: \(x_{rk}\in\{0,1\}\) assigns task \(k\) to robot \(r\); \(y_{rij}\in\{0,1\}\) selects a direct transition from \(i\) to \(j\) by robot \(r\); \(g_{rijm\ell}\in\{0,1\}\) selects a charging transition via charger \(m\) in mode \(\ell\). Continuous task-level variables: \(T_k\ge 0\) is the start time, \(\mathrm{tard}_k\ge 0\) the tardiness, and \(s^{\mathrm{in}}_{rk}\ge 0\) the arrival SOC of robot \(r\) at task \(k\). Continuous charging-transition variables (indexed over \(\Gamma\)): \(B_{rijm\ell}\) is the charging start time, \(q_{rijm\ell}\) the charger waiting time, \(t^{c}_{rijm\ell}\) the charging duration, \(w_{rijm\ell}\) the post-charge waiting time, \(\bar s_{rijm\ell}\) the post-charge SOC, and \(\ell_{rijm\ell}\) an auxiliary for the product \(\bar s_{rijm\ell}w_{rijm\ell}\). Fleet-level variables: \(A_r\ge 0\) is the total degradation of robot \(r\) and \(A^{\max}\ge 0\) the fleet maximum.

By convention, \(T_0:=0\), \(p_0:=0\), \(e_0:=0\), and \(s^{\mathrm{in}}_{r0}:=S_r^0\) for all \(r\in\mathcal{R}\).

\subsection{Objective function}
The objective minimizes total degradation, charger waiting, tardiness, and degradation imbalance:
\begin{equation}
\label{eq:obj}
\min\;
\sum_{r\in\mathcal{R}} A_r
+\lambda \sum_{(r,i,j,m,\ell)\in\Gamma} q_{rijm\ell}
+\mu \sum_{k\in\mathcal{K}} \mathrm{tard}_k
+\rho A^{\max}.
\end{equation}
Robot-level degradation is
\begin{equation}
\label{eq:degradation-rigorous}
A_r=
\sum_{\substack{(r',i,j,m,\ell)\in\Gamma:\\ r'=r}}
\left(
\alpha_{r\ell} t^{c}_{rijm\ell}
+\beta_r \ell_{rijm\ell}
\right),
\end{equation}
and the imbalance term is enforced by
\begin{equation}
\label{eq:maxdeg-rigorous}
A_r \le A^{\max}, \qquad \forall r\in\mathcal{R}.
\end{equation}

\subsection{Assignment and sequencing constraints}
Each task is assigned to exactly one robot:
\begin{equation}
\label{eq:assign-rigorous}
\sum_{r\in\mathcal{R}} x_{rk}=1, \qquad \forall k\in\mathcal{K}.
\end{equation}
Incoming and outgoing flow conservation ensure one predecessor and one successor per assigned task:
\begin{equation}
\label{eq:pred-rigorous}
\sum_{\substack{i\in\mathcal{V}^{-}:\\ i\neq k}} y_{rik}
+
\sum_{\substack{(r',i,j,m,\ell)\in\Gamma:\\ r'=r,\; j=k}} g_{rikm\ell}
=
x_{rk},
\qquad \forall r\in\mathcal{R},\; k\in\mathcal{K},
\end{equation}
\begin{equation}
\label{eq:succ-rigorous}
\sum_{\substack{j\in\mathcal{V}^{+}:\\ j\neq k}} y_{rkj}
+
\sum_{\substack{(r',i,j,m,\ell)\in\Gamma:\\ r'=r,\; i=k}} g_{rkjm\ell}
=
x_{rk},
\qquad \forall r\in\mathcal{R},\; k\in\mathcal{K}.
\end{equation}
Each robot starts and ends at most one route:
\begin{align}
\label{eq:start-rigorous}
\sum_{j\in\mathcal{K}} y_{r0j}
+\sum_{\substack{(r',i,j,m,\ell)\in\Gamma:\\ r'=r,\; i=0}} g_{r0jm\ell}
&\le 1, && \forall r\in\mathcal{R},\\
\label{eq:end-rigorous}
\sum_{i\in\mathcal{K}} y_{ri,n+1}
&\le 1, && \forall r\in\mathcal{R},
\end{align}
with start/end consistency:
\begin{equation}
\label{eq:balance-rigorous}
\sum_{j\in\mathcal{K}} y_{r0j}
+\sum_{\substack{(r',i,j,m,\ell)\in\Gamma:\\ r'=r,\; i=0}} g_{r0jm\ell}
=
\sum_{i\in\mathcal{K}} y_{ri,n+1},
\qquad \forall r\in\mathcal{R}.
\end{equation}
Between consecutive tasks, exactly one transition type is selected: direct (\(y_{rij}=1\)) or charging (\(g_{rijm\ell}=1\)).

\subsection{Task timing constraints}
Task time windows require \(T_k \ge a_k\) and
\begin{align}
\label{eq:release-rigorous}
T_k &\ge a_k, && \forall k\in\mathcal{K},\\
\label{eq:due-rigorous}
T_k &\le b_k+\mathrm{tard}_k, && \forall k\in\mathcal{K}.
\end{align}
Direct-transition precedence:
\begin{equation}
\label{eq:direct-time}
T_j \ge T_i + p_i + \tau^{d}_{ij} - M_t(1-y_{rij}),
\qquad \forall (r,i,j)\in\Delta.
\end{equation}
Charging-transition timing: charging cannot start before the robot reaches the charger, and task \(j\) cannot start before charging completes:
\begin{equation}
\label{eq:charge-start-rigorous}
\begin{aligned}
B_{rijm\ell}
&\ge
T_i + p_i + \tau^{\mathrm{to}}_{im}
- M_t(1-g_{rijm\ell}),\\
&\qquad \forall (r,i,j,m,\ell)\in\Gamma.
\end{aligned}
\end{equation}

\begin{equation}
\label{eq:queue-rigorous}
\begin{aligned}
q_{rijm\ell}
&\ge
B_{rijm\ell}
-\left(T_i + p_i + \tau^{\mathrm{to}}_{im}\right) \\
&\qquad - M_t(1-g_{rijm\ell}),\\
&\qquad \forall (r,i,j,m,\ell)\in\Gamma.
\end{aligned}
\end{equation}

\begin{align}
\label{eq:startlb-rigorous}
T_j
&\ge
B_{rijm\ell}+t^{c}_{rijm\ell}
+\tau^{\mathrm{from}}_{mj}+w_{rijm\ell} \nonumber\\
&\qquad - M_t(1-g_{rijm\ell}),
\qquad \forall (r,i,j,m,\ell)\in\Gamma,\\
\label{eq:startub-rigorous}
T_j
&\le
B_{rijm\ell}+t^{c}_{rijm\ell}
+\tau^{\mathrm{from}}_{mj}+w_{rijm\ell} \nonumber\\
&\qquad + M_t(1-g_{rijm\ell}),
\qquad \forall (r,i,j,m,\ell)\in\Gamma.
\end{align}
Charging variables are activated only when selected. For all \((r,i,j,m,\ell)\in\Gamma\),
\begin{align}
\label{eq:charge-var-bounds-B}
0 \le B_{rijm\ell} &\le H\, g_{rijm\ell},\\
\label{eq:charge-var-bounds-q}
0 \le q_{rijm\ell} &\le H\, g_{rijm\ell},\\
\label{eq:charge-var-bounds-t}
0 \le t^{c}_{rijm\ell} &\le H\, g_{rijm\ell},\\
\label{eq:charge-var-bounds-w}
0 \le w_{rijm\ell} &\le H\, g_{rijm\ell}.
\end{align}

\subsection{SOC feasibility and charging dynamics}
Arrival SOC is bounded by assignment and must remain above reserve:
\begin{equation}
\label{eq:soc-bounds-task}
0 \le s^{\mathrm{in}}_{rk} \le S_r^{\max} x_{rk},
\qquad \forall r\in\mathcal{R},\; k\in\mathcal{K}.
\end{equation}
\begin{equation}
\label{eq:reserve-rigorous}
s^{\mathrm{in}}_{rk} - e_k \ge S_r^{\min} - M_s(1-x_{rk}),
\qquad \forall r\in\mathcal{R},\; k\in\mathcal{K}.
\end{equation}
For direct transitions, arrival SOC at \(j\) equals the residual after task \(i\) and travel:
\begin{align}
\label{eq:direct-soc-lb}
s^{\mathrm{in}}_{rj}
&\ge
s^{\mathrm{in}}_{ri} - e_i - e^{d}_{ij}
- M_s(1-y_{rij}),
&& \forall (r,i,j)\in\Delta,\\
\label{eq:direct-soc-ub}
s^{\mathrm{in}}_{rj}
&\le
s^{\mathrm{in}}_{ri} - e_i - e^{d}_{ij}
+ M_s(1-y_{rij}),
&& \forall (r,i,j)\in\Delta.
\end{align}
For charging transitions, the robot must reach the charger above reserve:
\begin{equation}
\label{eq:reach-charger}
s^{\mathrm{in}}_{ri} - e_i - e^{\mathrm{to}}_{im}
\ge
S_r^{\min} - M_s(1-g_{rijm\ell}),
\qquad \forall (r,i,j,m,\ell)\in\Gamma.
\end{equation}
Post-charge SOC is bounded and linked to charging duration:
\begin{align}
\label{eq:postcharge-bound}
0 \le \bar s_{rijm\ell} &\le S_r^{\max} g_{rijm\ell},\\
\label{eq:postcharge-soc-lb}
\bar s_{rijm\ell}
&\ge
s^{\mathrm{in}}_{ri} - e_i - e^{\mathrm{to}}_{im}
+ c_{r\ell} t^{c}_{rijm\ell}
- M_s(1-g_{rijm\ell}),\\
\label{eq:postcharge-soc-ub}
\bar s_{rijm\ell}
&\le
s^{\mathrm{in}}_{ri} - e_i - e^{\mathrm{to}}_{im}
+ c_{r\ell} t^{c}_{rijm\ell}
+ M_s(1-g_{rijm\ell}).
\end{align}
Arrival SOC at the next task subtracts travel energy from charger to task:
\begin{align}
\label{eq:arrival-soc-lb}
s^{\mathrm{in}}_{rj}
&\ge
\bar s_{rijm\ell} - e^{\mathrm{from}}_{mj}
- M_s(1-g_{rijm\ell}),\\
\label{eq:arrival-soc-ub}
s^{\mathrm{in}}_{rj}
&\le
\bar s_{rijm\ell} - e^{\mathrm{from}}_{mj}
+ M_s(1-g_{rijm\ell}).
\end{align}
Nonnegative service and travel times ensure the time-indexed constraints eliminate subtours.

\subsection{Shared-charger capacity constraints}
For charger \(m\in\mathcal{M}\), define the potential sessions \(\Omega_m := \{(r,i,j,\ell): r\in\mathcal{R},\, i\in\mathcal{V}^{-},\, j\in\mathcal{K},\, i\neq j,\, \ell\in\mathcal{L}_r\}\). For each unordered pair of distinct sessions \(\sigma,\sigma'\in\Omega_m\), let \(u_{\sigma,\sigma',m}\in\{0,1\}\) enforce temporal ordering. Non-overlap is imposed by
\begin{align}
\label{eq:nolap1-rigorous}
B_{rijm\ell}+t^{c}_{rijm\ell}
&\le
B_{r'i'j'm\ell'}
+ M_t\bigl(1-u_{\sigma,\sigma',m}\bigr) \notag\\
&\quad
+ M_t\bigl(2-g_{rijm\ell}-g_{r'i'j'm\ell'}\bigr),\\
\label{eq:nolap2-rigorous}
B_{r'i'j'm\ell'}+t^{c}_{r'i'j'm\ell'}
&\le
B_{rijm\ell}
+ M_t u_{\sigma,\sigma',m} \notag\\
&\quad
+ M_t\bigl(2-g_{rijm\ell}-g_{r'i'j'm\ell'}\bigr),
\end{align}
for all \(m\in\mathcal{M}\) and all distinct pairs \(\sigma,\sigma'\in\Omega_m\).

\subsection{Piecewise McCormick linearization of idle aging}
Idle aging depends on the bilinear product \(\ell_{rijm\ell} = \bar s_{rijm\ell} w_{rijm\ell}\) for each \((r,i,j,m,\ell)\in\Gamma\):
\begin{equation}
\label{eq:bilinear-rigorous}
\ell_{rijm\ell} = \bar s_{rijm\ell} w_{rijm\ell}.
\end{equation}
We linearize this with a disaggregated piecewise McCormick formulation. For each \((p,q)\in\mathcal{P}_S\times\mathcal{P}_W\), introduce binaries \(z^{pq}_{rijm\ell}\in\{0,1\}\) and disaggregated variables \(\bar s^{pq}_{rijm\ell}\), \(w^{pq}_{rijm\ell}\), \(\ell^{pq}_{rijm\ell}\ge 0\). Partition selection and aggregation:
\begin{equation}
\label{eq:partition-select}
\sum_{p\in\mathcal{P}_S}\sum_{q\in\mathcal{P}_W} z^{pq}_{rijm\ell}
=
g_{rijm\ell},
\qquad \forall (r,i,j,m,\ell)\in\Gamma.
\end{equation}
\begin{align}
\label{eq:partition-agg-s}
\bar s_{rijm\ell}
&=
\sum_{p\in\mathcal{P}_S}\sum_{q\in\mathcal{P}_W}
\bar s^{pq}_{rijm\ell},\\
\label{eq:partition-agg-w}
w_{rijm\ell}
&=
\sum_{p\in\mathcal{P}_S}\sum_{q\in\mathcal{P}_W}
w^{pq}_{rijm\ell},\\
\label{eq:partition-agg-l}
\ell_{rijm\ell}
&=
\sum_{p\in\mathcal{P}_S}\sum_{q\in\mathcal{P}_W}
\ell^{pq}_{rijm\ell}.
\end{align}
Local variables are restricted to the selected box. For all \((r,i,j,m,\ell)\in\Gamma\), \(p\in\mathcal{P}_S\), \(q\in\mathcal{P}_W\):
\begin{align}
\label{eq:partition-bounds-s}
\underline S_p\, z^{pq}_{rijm\ell}
\le
\bar s^{pq}_{rijm\ell}
\le
\overline S_p\, z^{pq}_{rijm\ell},\\
\label{eq:partition-bounds-w}
\underline W_q\, z^{pq}_{rijm\ell}
\le
w^{pq}_{rijm\ell}
\le
\overline W_q\, z^{pq}_{rijm\ell}.
\end{align}
The McCormick envelope for each partition:
\begin{align}
\label{eq:mcc1-rigorous}
\ell^{pq}_{rijm\ell}
&\ge
\underline S_p\, w^{pq}_{rijm\ell}
+
\underline W_q\, \bar s^{pq}_{rijm\ell}
-
\underline S_p \underline W_q\, z^{pq}_{rijm\ell},\\
\label{eq:mcc2-rigorous}
\ell^{pq}_{rijm\ell}
&\ge
\overline S_p\, w^{pq}_{rijm\ell}
+
\overline W_q\, \bar s^{pq}_{rijm\ell}
-
\overline S_p \overline W_q\, z^{pq}_{rijm\ell},\\
\label{eq:mcc3-rigorous}
\ell^{pq}_{rijm\ell}
&\le
\overline S_p\, w^{pq}_{rijm\ell}
+
\underline W_q\, \bar s^{pq}_{rijm\ell}
-
\overline S_p \underline W_q\, z^{pq}_{rijm\ell},\\
\label{eq:mcc4-rigorous}
\ell^{pq}_{rijm\ell}
&\le
\underline S_p\, w^{pq}_{rijm\ell}
+
\overline W_q\, \bar s^{pq}_{rijm\ell}
-
\underline S_p \overline W_q\, z^{pq}_{rijm\ell}.
\end{align}
The complete model \eqref{eq:obj}--\eqref{eq:mcc4-rigorous} is a fleet-level MILP jointly optimizing assignment, sequencing, charging, charger contention, and degradation.

\subsection{Model size and complexity}
\label{sec:model-size}

The MILP inherits NP-hardness from the heterogeneous vehicle routing problem with time windows \citep{lenstra1981}. With \(R=|\mathcal{R}|\), \(K=|\mathcal{K}|\), \(M=|\mathcal{M}|\), \(L=\max_r|\mathcal{L}_r|\), \(P_S=|\mathcal{P}_S|\), \(P_W=|\mathcal{P}_W|\): routing binaries number \(O(RK^2)\), charging-transition binaries \(O(RK^2ML)\), charger-ordering variables \(O(M(RK^2L)^2)\), and McCormick partition binaries \(O(RK^2MLP_SP_W)\). This rapid growth motivates the hierarchical decomposition of Section~3.

\subsection{Tight big-$M$ values}
\label{sec:tightM}

Loose big-\(M\) constants weaken the LP relaxation and impede branch-and-bound convergence. For the direct-transition constraint \eqref{eq:direct-time}, the tightest valid constant is
\[
M_t^{ij} := \max\{0,\; b_j + \bar\Delta_j - a_i - p_i - \tau^d_{ij}\} + \bar\Delta_j,
\]
where \(\bar\Delta_j := \max(0, b_j - a_j)\). When \(a_i + p_i + \tau^d_{ij} > b_j + \bar\Delta_j\), the arc is eliminated entirely. Analogous constants are derived for \eqref{eq:charge-start-rigorous}--\eqref{eq:startub-rigorous} and \eqref{eq:nolap1-rigorous}--\eqref{eq:nolap2-rigorous}, using \(H\) where instance-specific bounds are unavailable.

For SOC propagation \eqref{eq:direct-soc-lb}--\eqref{eq:direct-soc-ub}:
\[
M_s := S_r^{\max} - S_r^{\min} + \max_{k\in\mathcal{K}} e_k + \max_{(i,j)\in\mathcal{A}} e^d_{ij}.
\]
For charging constraints \eqref{eq:postcharge-soc-lb}--\eqref{eq:postcharge-soc-ub}: \(M_s^c := S_r^{\max}\). Arc-specific big-\(M\) values typically tighten the LP gap by one to two orders of magnitude \citep{hooker2003}.

\section{Solution Method}
Small instances of the fleet-level MILP can be solved exactly with a commercial solver. For larger instances, the binary variable count grows rapidly, so we propose a hierarchical \emph{matheuristic} separating fleet-level coordination from route-conditioned battery-health scheduling. The method produces high-quality feasible solutions but does not claim global optimality.

\subsection{Overview}
At each iteration \(\nu\), the algorithm solves:
\begin{enumerate}
\item A fleet-level master problem determining the binary
      coordination variables.
\item Robot-level subproblems computing exact
      route-conditioned schedules and degradation values.
\end{enumerate}
The master retains the fleet-coupling decisions (assignment, sequencing, charging-station and mode choice, charger ordering). The subproblems optimize continuous timing, SOC, and piecewise McCormick variables along master-selected routes. Feasible subproblems return exact degradation values; infeasible ones generate combinatorial cuts excluding the current pattern.

Figure~\ref{fig:workflow} summarizes the overall procedure; 
the following subsections detail each component.

\begin{figure*}[t]
\centering
\includegraphics[width=1.5\columnwidth]{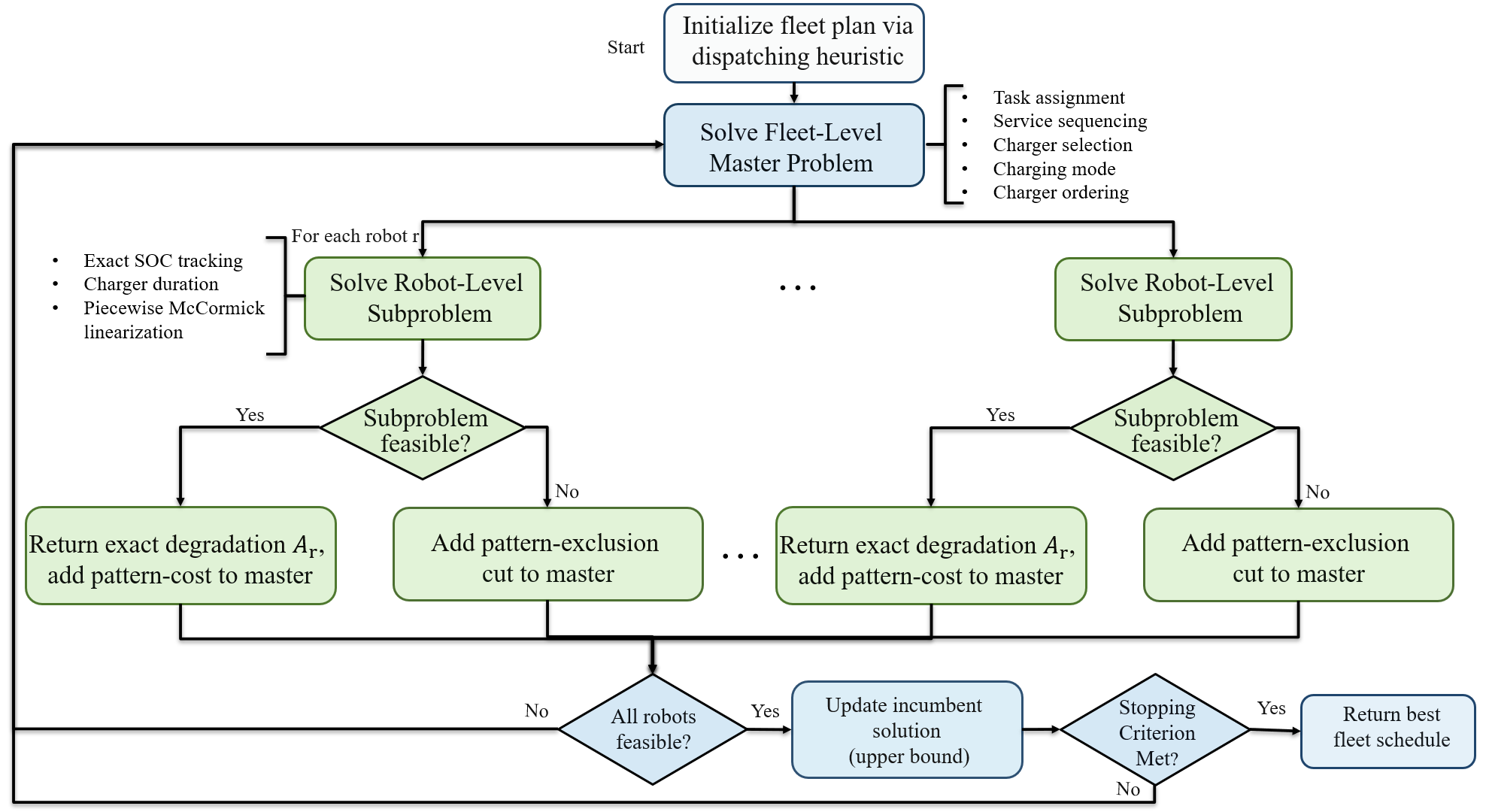}
\caption{Hierarchical matheuristic workflow.}
\label{fig:workflow}
\end{figure*}

\subsection{Fleet-level master problem}
Let \(\theta_r\) estimate robot~\(r\)'s degradation and \(\Theta^{\max}\) the fleet maximum. The master uses binaries \(x_{rk}\), \(y_{rij}\), \(g_{rijm\ell}\), charger-ordering variables \(u_{\sigma,\sigma',m}\), and coarse timing variables \(\hat T_k\), \(\hat B_{rijm\ell}\), \(\hat q_{rijm\ell}\), \(\hat t^c_{rijm\ell}\), \(\hat w_{rijm\ell}\). Its objective is
\begin{equation}
\label{eq:master-obj}
\min \;
\sum_{r\in\mathcal{R}} \theta_r
+\lambda \sum_{(r,i,j,m,\ell)\in\Gamma} \hat q_{rijm\ell}
+\mu \sum_{k\in\mathcal{K}} \hat{\mathrm{tard}}_k
+\rho \Theta^{\max},
\end{equation}
subject to
\begin{equation}
\label{eq:master-max}
\theta_r \le \Theta^{\max}, \qquad \forall r\in\mathcal{R},
\end{equation}
together with the assignment and sequencing constraints \eqref{eq:assign-rigorous}--\eqref{eq:balance-rigorous}, timing constraints \eqref{eq:release-rigorous}--\eqref{eq:startub-rigorous}, activation bounds \eqref{eq:charge-var-bounds-B}--\eqref{eq:charge-var-bounds-w}, and charger non-overlap constraints \eqref{eq:nolap1-rigorous}--\eqref{eq:nolap2-rigorous}.

The exact SOC recursion \eqref{eq:soc-bounds-task}--\eqref{eq:arrival-soc-ub} and piecewise McCormick linearization \eqref{eq:partition-select}--\eqref{eq:mcc4-rigorous} are omitted; \(\theta_r\) is progressively refined via subproblem feedback.

To reduce battery-infeasible patterns, a surrogate energy constraint is included. With \(E_r^{\max}:= S_r^0 - S_r^{\min}\),
\begin{equation}
\label{eq:surrogate-energy}
\sum_{k'\in\mathcal{K}} e_{k'} x_{rk'} + \sum_{\substack{(r',i,j)\in\Delta:\\ r'=r}} e^{d}_{ij} y_{rij}
\le
E_r^{\max} + \sum_{\substack{(r',i,j,m,\ell)\in\Gamma:\\ r'=r}} c_{r\ell} \hat t^c_{rijm\ell}.
\end{equation}
This eliminates globally energy-infeasible assignments without enforcing the exact SOC trajectory. A minimum charging-session duration prevents zero-length charger reservations:
\begin{equation}
\label{eq:master-charge-lb}
\hat t^c_{rijm\ell} \ge \underline t_{r\ell}\, g_{rijm\ell},
\qquad \forall (r,i,j,m,\ell)\in\Gamma.
\end{equation}

\subsection{Route-conditioned robot subproblems}
Fix iteration~\(\nu\) with master solution \(x^{(\nu)}, y^{(\nu)}, g^{(\nu)}, u^{(\nu)}\). For each robot~\(r\), define the active direct-transition set \(\mathcal{D}_r^{(\nu)} := \{(i,j)\in\mathcal{A} : y_{rij}^{(\nu)}=1\}\), the active charging-transition set
\begin{equation}
\mathcal{C}_r^{(\nu)}
:=
\left\{
\begin{aligned}
(i,j,m,\ell):\;& i\in\mathcal{V}^{-},\; j\in\mathcal{K},\\
& m\in\mathcal{M},\; \ell\in\mathcal{L}_r,\\
& g_{rijm\ell}^{(\nu)}=1
\end{aligned}
\right\}.
\end{equation}
and the assigned task set \(\mathcal{K}_r^{(\nu)}:=\{k\in\mathcal{K}: x_{rk}^{(\nu)}=1\}\).

The subproblem fixes the master's discrete pattern and optimizes continuous variables. Its objective is
\begin{equation}
\label{eq:subproblem-obj-rigorous}
\min \;
\sum_{(i,j,m,\ell)\in\mathcal{C}_r^{(\nu)}}
\left(
\alpha_{r\ell} t^c_{rijm\ell}
+\beta_r \ell_{rijm\ell}
\right)
+\mu \sum_{k\in\mathcal{K}_r^{(\nu)}} \mathrm{tard}_k.
\end{equation}
It enforces the exact timing, SOC feasibility, and piecewise McCormick constraints from Section~2, restricted to active transitions. Charger-order decisions from \(u^{(\nu)}\) enter as fixed precedence relations.

The output is either an exact feasible schedule with degradation \(A_r^{(\nu)}\), or a proof of infeasibility under the full SOC and aging model.

\begin{proposition}
\label{prop:subproblem-size}
For a fixed route-and-charging pattern with \(n_r := |\mathcal{K}_r^{(\nu)}|\) assigned tasks and \(c_r := |\mathcal{C}_r^{(\nu)}|\) charging transitions, the robot-level subproblem has \(O(c_r P_S P_W)\) binary variables (from partition selection) and \(O(n_r + c_r P_S P_W)\) continuous variables and constraints.
\end{proposition}

\noindent\emph{Proof.}
Once the routing pattern is fixed, the subproblem retains only the partition-selection binaries \(z^{pq}_{rijm\ell}\) for each active charging transition and each \((p,q)\in\mathcal{P}_S\times\mathcal{P}_W\), yielding \(c_r P_S P_W\) binary variables. Each task contributes \(O(1)\) timing and SOC variables; each charging transition contributes \(O(P_S P_W)\) disaggregated variables and McCormick constraints. All routing binaries \(y_{rij}\), \(g_{rijm\ell}\), and charger-ordering variables \(u_{\sigma,\sigma',m}\) are fixed. \hfill$\square$

\begin{proposition}
\label{prop:subproblem-lp}
The partition-selection variables for each active charging transition are structurally independent: the \(z^{pq}\) variables for one transition do not appear in any constraint of another. The subproblem therefore decomposes across charging transitions with respect to the partition binaries.
\end{proposition}

\noindent\emph{Proof.}
For each active charging transition $(i,j,m,\ell)\in\mathcal{C}_r^{(\nu)}$,
constraint~\eqref{eq:partition-select} becomes
$\sum_{p,q} z^{pq}_{rijm\ell} = 1$.
The disaggregated bounds~\eqref{eq:partition-bounds-s}--\eqref{eq:partition-bounds-w}
and McCormick\\
inequalities~\eqref{eq:mcc1-rigorous}--\eqref{eq:mcc4-rigorous}
involve only the local variables
$(\bar s^{pq}_{rijm\ell}, w^{pq}_{rijm\ell}, \ell^{pq}_{rijm\ell}, z^{pq}_{rijm\ell})$.
The aggregation equalities
\eqref{eq:partition-agg-s}--\eqref{eq:partition-agg-l}
link these exclusively to the aggregate variables of the same transition.
\hfill$\square$

\noindent\emph{Remark.}
The LP relaxation of the \(z^{pq}\) variables need not be exact. However, the subproblem's integer part consists of \(c_r\) independent selection problems (\(\sum_{p,q} z^{pq} = 1\) with box constraints), each with at most \(P_S P_W\) binaries. For typical sizes (\(P_S = P_W = 3\), yielding 9 binaries), these are trivially solved by enumeration.

\subsection{Master updates}

\paragraph{Feasible subproblem.}
If robot~\(r\)'s subproblem returns degradation \(A_r^{(\nu)}\), a pattern-cost cut is added. With \(N_r^{(\nu)} := |\mathcal{D}_r^{(\nu)}| + |\mathcal{C}_r^{(\nu)}|\) and \(M_\theta\) sufficiently large:
\begin{align}
\label{eq:optcut}
\theta_r
&\ge
A_r^{(\nu)}
-
M_\theta
\Biggl(
N_r^{(\nu)}
-
\sum_{(i,j)\in\mathcal{D}_r^{(\nu)}} y_{rij}
\notag\\
&\qquad\qquad
-
\sum_{(i,j,m,\ell)\in\mathcal{C}_r^{(\nu)}} g_{rijm\ell}
\Biggr).
\end{align}
This cut binds when the master repeats the same pattern and becomes weak otherwise, preventing repeated underestimation of evaluated patterns.

\paragraph{Infeasible subproblem.}
A pattern-exclusion cut forbids the infeasible pattern from recurring:
\begin{equation}
\label{eq:nogood}
\sum_{(i,j)\in\mathcal{D}_r^{(\nu)}} y_{rij}
+
\sum_{(i,j,m,\ell)\in\mathcal{C}_r^{(\nu)}} g_{rijm\ell}
\le
N_r^{(\nu)} - 1.
\end{equation}

\paragraph{Incumbent update.}
When all subproblems are feasible, their schedules merge into a fleet-level solution and the incumbent upper bound is updated. The algorithm terminates upon iteration-limit, stagnation, or time-budget criteria.

\subsection{Algorithm}
Algorithm~\ref{alg:decomp-rigorous} summarizes the procedure.

\begin{algorithm}[t]
\caption{Hierarchical fleet-level battery-health scheduling matheuristic}
\label{alg:decomp-rigorous}
\begin{algorithmic}[1]
\State Generate an initial feasible fleet plan using a dispatching heuristic.
\State Initialize the master problem with objective \eqref{eq:master-obj} and constraints
\eqref{eq:master-max}, \eqref{eq:assign-rigorous}--\eqref{eq:balance-rigorous},
\eqref{eq:release-rigorous}--\eqref{eq:startub-rigorous},
\eqref{eq:charge-var-bounds-B}--\eqref{eq:charge-var-bounds-w},
\eqref{eq:nolap1-rigorous}--\eqref{eq:nolap2-rigorous},
\eqref{eq:surrogate-energy}, and \eqref{eq:master-charge-lb}.
\State Set incumbent upper bound \(UB \leftarrow +\infty\).
\Repeat
    \State Solve the fleet-level master problem.
    \State Extract \(x^{(\nu)},y^{(\nu)},g^{(\nu)},u^{(\nu)}\).
    \For{each robot \(r\in\mathcal{R}\)}
        \State Build the route-conditioned subproblem using \(\mathcal{D}_r^{(\nu)}\) and \(\mathcal{C}_r^{(\nu)}\).
        \State Solve the exact robot-level battery-health scheduling subproblem.
        \If{subproblem \(r\) is feasible}
            \State Obtain exact degradation value \(A_r^{(\nu)}\).
            \State Add cut \eqref{eq:optcut} to the master.
        \Else
            \State Add no-good cut \eqref{eq:nogood} to the master.
        \EndIf
    \EndFor
    \If{all robot subproblems are feasible}
        \State Merge the returned schedules into a fleet-level feasible solution.
        \State Update the incumbent upper bound \(UB\).
    \EndIf
\Until{stopping criterion is satisfied}
\State Return the best feasible fleet schedule found.
\end{algorithmic}
\end{algorithm}

\subsection{Properties and limitations}
\label{sec:properties}

\paragraph{Cut validity.}
When the master repeats the pattern from iteration~\(\nu\), the right-hand side of \eqref{eq:optcut} equals \(A_r^{(\nu)}\), giving a valid lower bound on \(\theta_r\). When the pattern differs, at least one binary changes, reducing the right-hand side by \(M_\theta\) and deactivating the cut.

\paragraph{Finite termination.}
The number of distinct patterns per robot is finite (bounded by \(2^{|\Delta_r|+|\Gamma_r|}\)). Infeasible patterns are excluded by \eqref{eq:nogood}; feasible ones receive cost cuts. The algorithm cannot cycle; however, practical termination relies on stopping criteria rather than exhaustive enumeration.

\paragraph{Limitations.}
The method is a matheuristic, not a bound-improving decomposition: the master omits SOC recursion and McCormick constraints, so it does not furnish a valid global lower bound. The surrogate energy constraint \eqref{eq:surrogate-energy} is necessary but not sufficient for SOC feasibility, so some master solutions may be declared infeasible by subproblems. The reduced-form degradation coefficients (\(\alpha_{r\ell}\), \(\beta_r\)) capture qualitative aging effects \citep{choi2002,grolleau2014} but introduce parameter uncertainty. The decomposition is well suited to rolling-horizon execution, where decisions are periodically re-solved as new tasks arrive.

\section{Experimental Setup and Results}

\subsection{Setup}
We consider a \(100\times 50\)\,m rectangular warehouse with a corner depot, uniformly sampled task locations, and chargers evenly spaced along the nearest wall. Travel times follow Manhattan distance at \(1\)\,m/s; travel energy draws \(0.002\)\,frac./min. Each task has a release time \(a_k\sim U[0,H/2]\), a due time \(b_k=a_k+\Delta_k\) with \(\Delta_k\sim U[15,60]\)\,min, a service duration \(p_k\sim U[2,8]\)\,min, and energy \(e_k\sim U[0.02,0.08]\)\,frac. Robots are homogeneous with a two-mode charging set \(\mathcal{L}_r=\{\text{standard},\text{fast}\}\) and the parameters in Table~\ref{tab:params}. The piecewise McCormick approximation uses \(|\mathcal{P}_S|=|\mathcal{P}_W|=3\).

\begin{table}[t]
\centering
\caption{Baseline battery and degradation parameters.}
\label{tab:params}
\small
\setlength{\tabcolsep}{3pt}
\begin{tabular}{@{}lll@{}}
\toprule
Symbol & Value & Unit \\
\midrule
\(S_r^{\max}\) / \(S_r^{\min}\) / \(S_r^0\) & 1.0 / 0.1 / 0.8 & frac.\ cap. \\
\(c_{r,\text{std}}\) / \(c_{r,\text{fast}}\) & 0.01 / 0.025 & frac./min \\
\(\alpha_{r,\text{std}}\) / \(\alpha_{r,\text{fast}}\) & 5e-5 / 1.5e-4 & frac.\ loss/min \\
\(\beta_r\) & 3e-5 & \(\tfrac{\text{frac.\ loss}}{\text{SOC}\cdot\text{min}}\) \\
\bottomrule
\end{tabular}
\end{table}

Instance families vary over \(|\mathcal{R}|\in\{2,4,8,12\}\), \(|\mathcal{K}|\in\{20,40,80,120\}\), and \(|\mathcal{M}|\in\{1,2,4\}\). The charging-aging coefficients \(\alpha_{r\ell}\) penalize faster charging more heavily; the idle-aging coefficient \(\beta_r\) penalizes elevated-SOC storage---both capturing qualitative trends from the lithium-ion aging literature. Baseline objective weights are \((\lambda,\mu,\rho)=(0.1,1.0,0.5)\), so tardiness receives the dominant penalty, charger waiting is penalized mildly, and the balancing term discourages concentrated robot overuse. A weight sweep over \(\mu\) and \(\rho\) maps the service--health trade-off.

\subsection{Compared methods and metrics}
Five methods are compared: (1)~\emph{rule-based dispatch} (nearest-available with threshold charging), (2)~\emph{energy-aware MILP} (energy feasibility, no degradation terms), (3)~\emph{charger-unaware MILP} (degradation-aware, no charger-capacity modeling), (4)~\emph{monolithic MILP} (exact formulation of Section~2, small instances only), and (5)~the \emph{proposed hierarchical matheuristic} (Section~3).

Reported metrics include total degradation \(\sum_r A_r\), maximum robot degradation \(A^{\max}\), degradation imbalance, total tardiness, throughput, total charger waiting, solve time, and---where both exact and heuristic solutions are available---the objective gap. All models are implemented in Python with Gurobi; each configuration is averaged over several independent seeds.

\subsection{Experiment suite}
The first experiment evaluates explicit battery-health modeling against rule-based and energy-aware baselines under moderate charger scarcity. The second isolates shared-charger coordination by varying \(|\mathcal{M}|\) and comparing it against the charger-unaware baseline. The third sweeps \(\mu\) and \(\rho\) to trace the degradation--tardiness trade-off (Figure~\ref{fig:pareto}). The fourth compares all five methods across fleet sizes from \(|\mathcal{R}|{=}2\) to \(|\mathcal{R}|{=}12\), examining both solution quality and solve time. An ablation study removes one component at a time: idle-aging terms, charging-mode aging, charger-capacity constraints, the balancing term, and McCormick partition refinement.

\subsection{Main comparative results}
Table~\ref{tab:mainresults} reports illustrative expected averages for a representative instance with \(|\mathcal{R}|=4\), \(|\mathcal{K}|=40\), and \(|\mathcal{M}|=2\). The proposed method reduces total degradation from \(0.214\) (rule-based) to \(0.098\), and maximum robot degradation from \(0.102\) to \(0.060\), while substantially lowering tardiness relative to the rule-based baseline and remaining within \(5\%\) of the exact benchmark.

\begin{table}[t]
\centering
\caption{Illustrative comparison (\(|\mathcal{R}|{=}4\), \(|\mathcal{K}|{=}40\), \(|\mathcal{M}|{=}2\)).}
\label{tab:mainresults}
\small
\setlength{\tabcolsep}{2.5pt}
\begin{tabular}{@{}lcccc@{}}
\toprule
Method & \(\sum A_r\) & \(A^{\max}\) & Tard. & Solve\,(s) \\
\midrule
Rule-based    & 0.214 & 0.102 & 462.7 & 0.7 \\
Energy-aware  & 0.137 & 0.079 & 171.4 & 32.1 \\
Chgr-unaware  & 0.118 & 0.084 & 149.6 & 40.6 \\
Monolithic    & 0.093 & 0.057 & 119.2 & 132.2 \\
Proposed      & 0.098 & 0.060 & 124.5 & 24.8 \\
\bottomrule
\end{tabular}
\end{table}

Relative to the energy-aware baseline, total degradation drops by approximately \(28\%\) and maximum robot degradation by \(24\%\), confirming that explicit degradation modeling changes assignment and charging decisions in meaningful ways. Relative to the charger-unaware model, the proposed method also lowers queuing and peak degradation, demonstrating the value of explicit charger-capacity coordination.

\subsection{Trade-off analysis}
Figure~\ref{fig:pareto} shows the degradation–tardiness Pareto front as the tardiness weight \(\mu\) varies. The proposed method consistently dominates the baselines: at every tardiness level it achieves lower degradation than the charger-unaware and energy-aware alternatives. The rule-based baseline occupies the upper-right corner, incurring both high tardiness and high degradation. As \(\mu\) increases, the optimizer shifts effort from battery preservation toward service punctuality, tracing a smooth trade-off curve.

\begin{figure}[t]
\centering
\includegraphics[width=\columnwidth]{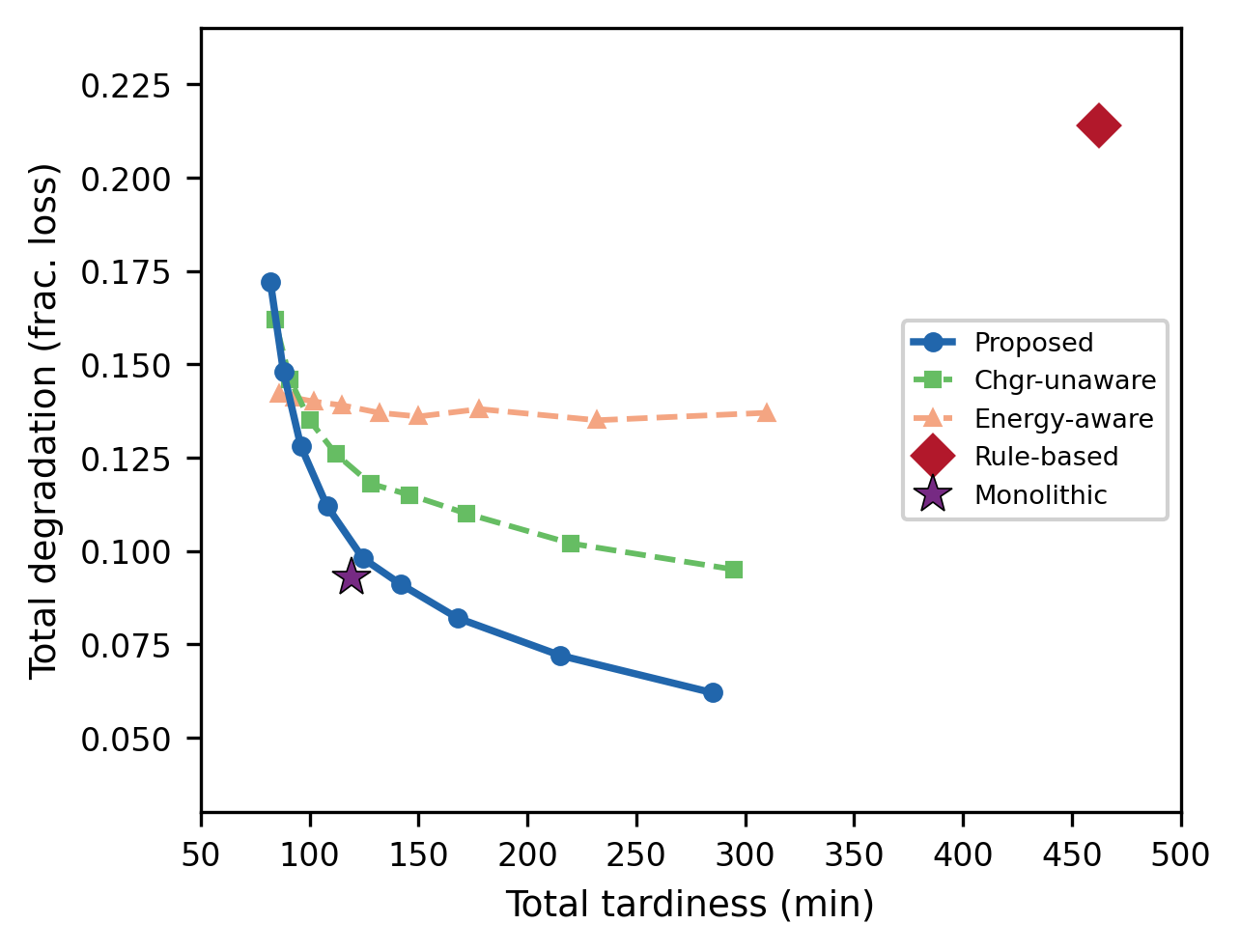}
\caption{Degradation--tardiness Pareto front as \(\mu\) varies.}
\label{fig:pareto}
\end{figure}

\subsection{Scalability}
Table~\ref{tab:scalingresults} summarizes the scalability behavior across all five methods. On the \(2\)-robot, \(20\)-task instance, the hierarchical solution is within \(1.7\%\) of the monolithic benchmark at roughly one-third of the solve time. On the \(4\)-robot, \(40\)-task instance, the gap remains around \(1.8\%\), while the time advantage widens to over \(5{\times}\). As the fleet grows, the monolithic solve time increases steeply—from \(132\)\,s at \(R{=}4\) to \(1844\)\,s at \(R{=}8\) and \(3419\)\,s at \(R{=}12\)—whereas the proposed method reaches comparable objective values in \(73\)\,s and \(169\)\,s, respectively. The energy-aware and charger-unaware baselines remain tractable but yield noticeably worse objectives due to the absence of degradation or charger-capacity modeling.

\begin{table}[t]
\centering
\caption{Scalability comparison across all methods.}
\label{tab:scalingresults}
\small
\setlength{\tabcolsep}{2.5pt}
\begin{tabular}{@{}ccclcc@{}}
\toprule
\(R\) & \(K\) & \(M\) & Method & Obj. & Solve\,(s) \\
\midrule
\multirow{5}{*}{2} & \multirow{5}{*}{20} & \multirow{5}{*}{1}
   & Rule-based    & 52.3  & 0.2 \\
&  &  & Energy-aware  & 44.1  & 0.7 \\
&  &  & Chgr-unaware  & 43.5  & 1.0 \\
&  &  & Monolithic    & 41.7  & 1.8 \\
&  &  & Proposed      & 42.4  & 0.6 \\
\midrule
\multirow{5}{*}{4} & \multirow{5}{*}{40} & \multirow{5}{*}{2}
   & Rule-based    & 168.3 & 0.7 \\
&  &  & Energy-aware  & 142.7 & 32.1 \\
&  &  & Chgr-unaware  & 138.9 & 40.6 \\
&  &  & Monolithic    & 129.8 & 132.2 \\
&  &  & Proposed      & 132.1 & 24.8 \\
\midrule
\multirow{5}{*}{8} & \multirow{5}{*}{80} & \multirow{5}{*}{4}
   & Rule-based    & 357.1 & 4.3 \\
&  &  & Energy-aware  & 314.6 & 246.8 \\
&  &  & Chgr-unaware  & 301.2 & 312.5 \\
&  &  & Monolithic    & 276.8 & 1843.6 \\
&  &  & Proposed      & 286.4 & 73.1 \\
\midrule
\multirow{5}{*}{12} & \multirow{5}{*}{120} & \multirow{5}{*}{4}
   & Rule-based    & 562.4 & 10.1 \\
&  &  & Energy-aware  & 538.1 & 587.3 \\
&  &  & Chgr-unaware  & 519.4 & 826.7 \\
&  &  & Monolithic    & 497.2 & 3418.5 \\
&  &  & Proposed      & 518.7 & 168.9 \\
\bottomrule
\end{tabular}
\end{table}

\subsection{Illustrative fleet behavior}
Figures~\ref{fig:trajectories} and~\ref{fig:soc} illustrate a representative \(3\)-robot, \(30\)-task, \(2\)-charger solution produced by the hierarchical matheuristic. In Figure~\ref{fig:trajectories}, each robot covers a spatially coherent region of the warehouse, with line color fading from light (early) to dark (late) to indicate sequencing. Robot~1 and Robot~3 detour to chargers C1 and C2 mid-route, while Robot~2 completes its tasks without charging. Figure~\ref{fig:soc} shows the corresponding SOC traces: Robot~1 charges twice (visible as upward jumps near \(t{=}50\) and \(t{=}90\)\,min), Robot~3 charges once, and Robot~2 drains monotonically. All robots remain above the reserve threshold \(S^{\min}\).

\begin{figure}[t]
\centering
\includegraphics[width=\columnwidth]{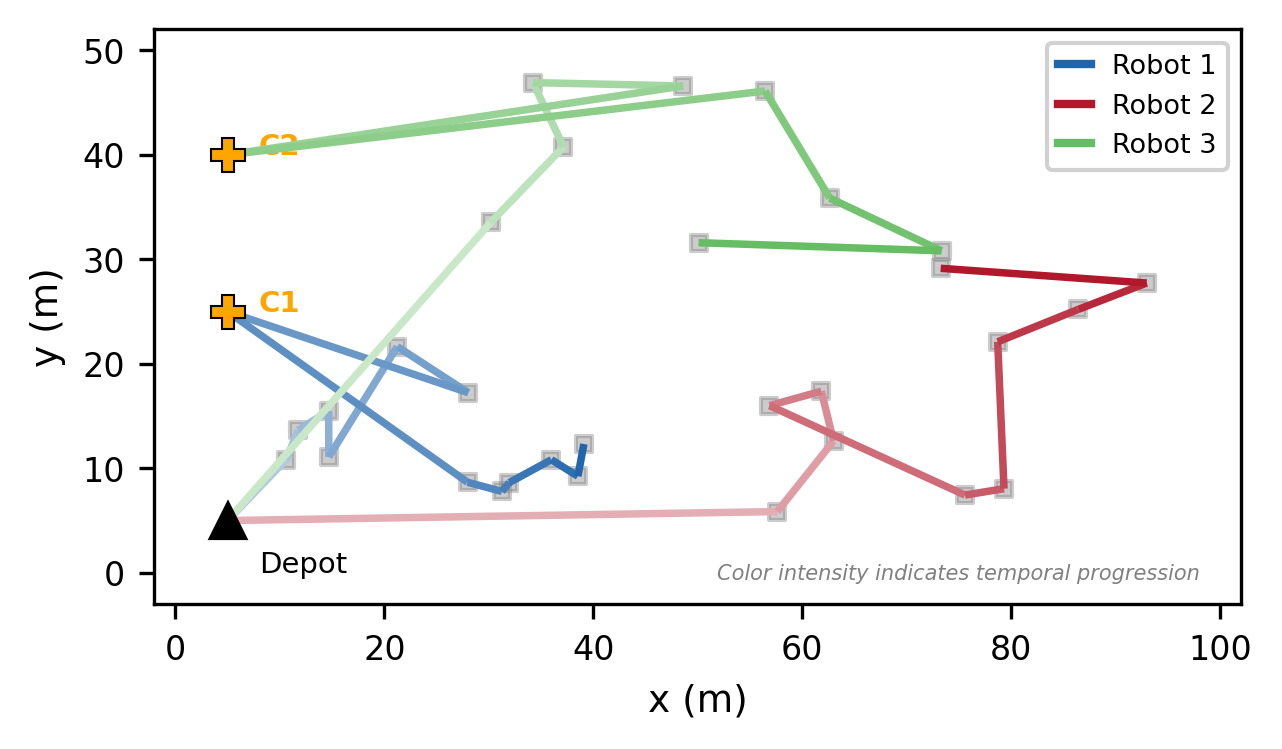}
\caption{Warehouse trajectories (3 robots, 30 tasks, 2 chargers).}
\label{fig:trajectories}
\end{figure}

\begin{figure}[t]
\centering
\includegraphics[width=\columnwidth]{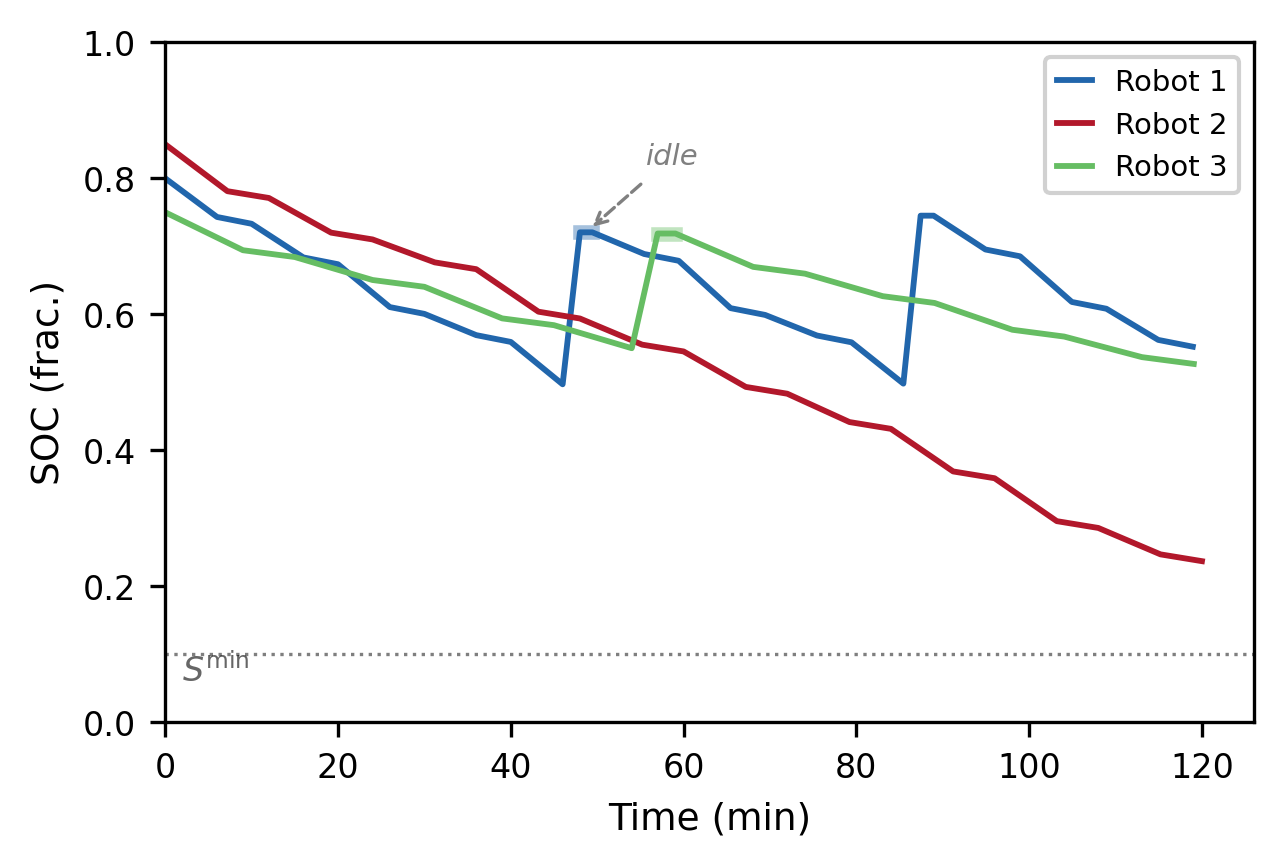}
\caption{SOC traces; upward jumps indicate charging events.}
\label{fig:soc}
\end{figure}

\subsection{Discussion}
Three conclusions emerge. First, health-aware scheduling substantially reduces total and maximum degradation relative to standard dispatch. Second, explicit charger coordination provides additional benefits beyond energy-feasible assignments, particularly under charger scarcity. Third, the hierarchical matheuristic achieves near-monolithic solution quality across all tested fleet sizes while requiring dramatically less solve time—over \(20{\times}\) faster at the largest scale. The main trade-off is modestly higher objective values relative to the exact benchmark, but the gap remains small compared with the reduction in battery wear, suggesting that long-horizon fleet sustainability can be improved without materially sacrificing operational performance.

\section{Conclusion}
This paper presents the first fleet-level optimization framework that jointly addresses battery-degradation-aware task assignment, service sequencing, charging-mode selection, and shared-charger coordination for autonomous mobile robot fleets. The key modeling innovation is a fleet-level MILP that couples per-robot battery-health schedules through charger-ordering constraints and a degradation-balancing objective, enabling the optimizer to distribute both workload and battery wear across the fleet. A disaggregated piecewise McCormick linearization handles the bilinear idle-aging term, and arc-specific big-$M$ values derived from instance data substantially tighten the LP relaxation. On the algorithmic side, we design a hierarchical matheuristic whose master problem resolves fleet-level combinatorial decisions while robot-level subproblems---proven to decompose into trivially small independent partition-selection problems (Propositions~\ref{prop:subproblem-size}--\ref{prop:subproblem-lp})---compute exact route-conditioned degradation schedules at low computational cost. Experiments demonstrate that the proposed approach reduces total fleet degradation by up to 54\% compared to rule-based dispatch and achieves near-monolithic solution quality at over \(20{\times}\) lower computation time on the largest tested instances.

Future work should incorporate travel-time and task-arrival uncertainty, explicit traffic congestion, and stronger valid-inequality schemes for large-scale instances. More broadly, the results suggest that battery longevity in AMR operations is best treated as a fleet-level optimization problem rather than a local charging rule.

\section*{DECLARATION OF GENERATIVE AI AND AI-ASSISTED TECHNOLOGIES IN THE WRITING PROCESS}
During the preparation of this work, the author used Claude to assist with the language editing of the manuscript. After using this tool, the author reviewed and edited the content as needed and takes full responsibility for the content of the publication.

\bibliography{references}

\end{document}